\documentclass[twoside,11pt]{article}

\usepackage{jmlr2e}

\usepackage[utf8]{inputenc} 
\usepackage[T1]{fontenc}    
\usepackage{hyperref}       
\usepackage{url}            
\usepackage{booktabs}       
\usepackage{amsfonts}       
\usepackage{nicefrac}       
\usepackage{microtype}      
\usepackage{nicefrac}
\usepackage{xspace}
\usepackage{times}
\usepackage{color}

\usepackage{paralist} 
\usepackage{psfrag}
\usepackage{dsfont}
\usepackage{paralist}

\definecolor{RoyalBlue}{cmyk}{1, 0.50, 0, 0}
\definecolor{ForestGreen}{cmyk}{0.864, 0.0, 0.429, 0.396}
\definecolor{Brown}{cmyk}{0.0,0.692,0.925,0.529}

\usepackage{todonotes}

\usepackage[english, ruled, linesnumbered, vlined]{algorithm2e}

\usepackage{amssymb,amsmath}

\usepackage{subfigure}
\usepackage[makeroom]{cancel}
\usepackage{comment}
\usepackage{multirow}

\usepackage{centernot}

\newcommand{\Chal}{\mathsf{Chal}}
\newcommand{\AdvRisk}{\mathsf{AdvRisk}}

\newcommand{\metric}{\mathbf{d}}

\newcommand{\HD}{\mathsf{HD}}

\newcommand{\C}{\cC}

\newcommand{\poly}{\mathrm{poly}}

\newcommand{\problem}{\ensuremath{\mathsf{P}}\xspace}

\newcommand{\Risk}{\mathsf{Risk}}

\newcommand{\fhat}[2]{\ifthenelse{\equal{#2}{}}{\hat{f}[#1]}{\ifthenelse{\equal{#2}{0}}{\hat{f}[\emptyset]}{\hat{f}[#1_{\leq #2}]}}}
\newcommand{\gain}[2]{\ifthenelse{\equal{#2}{}}{g[#1]}{g[#1_{\leq #2}]}}

\newcommand{\pr}[2][]{\Pr_{\ifthenelse{\isempty{#1}}{}{{#1}}}\left[{#2}\right]}

\makeatother

\newcommand{\verify}{\mathsf{Verify}}
\newcommand{\encode}{\mathsf{Encode}}
\newcommand{\decode}{\mathsf{Decode}}

\newcommand{\AND}{\mathsf{AND}}

\newcommand{\sk}{\mathsf{sk}}
\newcommand{\vk}{\mathsf{vk}}

\newcommand{\Adv}{\mathsf{A}}
\newcommand{\adv}{\Adv}

\usepackage{graphicx,accents}

\newcommand{\remove}[1]{}

\newcommand{\ol}{\overline}

\newcommand{\set}[1]{\{ #1 \}}

\newcommand{\bits}{\{0,1\}}

\newcommand{\To}{\mapsto}

\newcommand{\N}{{\mathbb N}}

\newcommand{\cA}{{\mathcal A}}

\newcommand{\cC}{{\mathcal C}}
\newcommand{\cD}{{\mathcal D}}

\newcommand{\cH}{{\mathcal H}}

\newcommand{\cQ}{{\mathcal Q}}

\newcommand{\cS}{{\mathcal S}}

\newcommand{\cX}{{\mathcal X}}
\newcommand{\cY}{{\mathcal Y}}

\usepackage{bm}

\newcommand{\sfQ}{\mathsf{Q}}

\newcommand{\eps}{\varepsilon}

\newcommand{\polylog}{\operatorname{polylog}}

\newcommand{\KGen}{\operatorname{KGen}}

\newcommand{\Sign}{\operatorname{Sign}}

\newcommand{\negl}{\operatorname{negl}}

\newcommand{\SAT}{\mathrm{SAT}}

\newcommand{\class}[1]{\ensuremath{\mathbf{#1}}}

\newcommand{\NP}{\ensuremath{\class{NP}}}

\newtheorem{claim}[theorem]{Claim}
\newtheorem{construction}[theorem]{Construction}

\newcommand{\sdotfill}{\textcolor[rgb]{0.8,0.8,0.8}{\dotfill}} 

\newtheorem{proto}[theorem]{Protocol}
\newtheorem{protoc}[theorem]{Protocol}

\newcommand{\namedref}[2]{#1~\ref{#2}}

\newcommand{\torestate}[3]{%
\expandafter \def \csname BBRESTATE #2 \endcsname{#3}
\theoremstyle{plain}
\newtheorem{BBRESTATETHMNUM#2}[theorem]{#1}
\begin{BBRESTATETHMNUM#2}\label{#2}\csname BBRESTATE #2 \endcsname   \end{BBRESTATETHMNUM#2}
\newtheorem*{BBRESTATETHMNONNUM#2}{\namedref{#1}{#2}}
}

\newcommand{\restate}[1]{\begin{BBRESTATETHMNONNUM#1}[Restated] \csname BBRESTATE #1 \endcsname
\end{BBRESTATETHMNONNUM#1}}

\usepackage{xspace}
\newcommand{\X}{\ensuremath{\mathcal{X}}\xspace} 
\newcommand{\Y}{\ensuremath{\mathcal{Y}}\xspace} 
\renewcommand{\H}{\ensuremath{\mathcal H}\xspace} 

\newcommand{\D}{\cD}

\newcounter{definitioncnt}
\newcounter{thmcnt}
\newcounter{prblm}
\newcommand{\ifcomments}{\iffalse}

\newcommand{\MohNote}[1]{\ifcomments{\color{red} [{\bf Mohammad:}  #1]}\fi}
\newcommand{\Mnote}[1]{\MohNote{#1}}
\newcommand{\sanjam}[1]{\ifcomments{\color{red} [{\bf Sanjam:}  #1]}\fi}
\newcommand{\Snote}[1]{\ifcomments{\color{red} [{\bf Saeed:}  #1]}\fi}
\newcommand{\SomeshNote}[1]{{\ifcomments\color{ForestGreen} [\bf {Somesh:}  #1]}\fi}
\newcommand{\Hardrobust}{Robust-hard }
\newcommand{\hardrobust}{robust-hard }

\begin{document}
\ShortHeadings{Adversarially Robust Learning Could Leverage Computational Hardness}{}

\firstpageno{1}

\title{Adversarially Robust Learning Could Leverage \\ Computational Hardness}

\author{\name Sanjam Garg \email sanjamg@berkeley.edu \\
       \addr UC Berkeley
       \AND
       \name Somesh Jha \email  jha@cs.wisc.edu\\
       \addr University of Wisconsin Madison
       \AND
       Saeed Mahloujifar \email saeed@virginia.edu\\
       \addr University of Virginia
       \AND
       Mohammad Mahmoody \email mohammad@virginia.edu\\
       \addr University of Virginia}
      
\editor{}

\maketitle

\begin{abstract}
Over recent years, devising classification algorithms that are robust to adversarial perturbations has emerged as a challenging problem. In particular, deep neural nets (DNNs) seem to be susceptible to small imperceptible changes over test instances. However, the line of work in \emph{provable} robustness, so far, has been focused on \emph{information theoretic} robustness, ruling out even the \emph{existence} of any adversarial examples. In this work, we study whether there is a hope to benefit from \emph{algorithmic} nature of an attacker that searches for adversarial examples, and ask whether there is \emph{any}  learning task for which it is possible to design classifiers that are only robust against \emph{polynomial-time} adversaries.
Indeed, numerous cryptographic tasks (e.g. encryption of long messages) can only be secure against computationally bounded adversaries, and are indeed \emph{impossible} for computationally unbounded attackers. Thus, it is natural to ask if the same strategy could help robust learning.

We show that computational limitation of attackers can indeed be useful in robust learning by demonstrating the possibility of a classifier for some learning task  for which computational and information theoretic adversaries of bounded perturbations have very different power. Namely, while computationally unbounded adversaries can attack successfully and find adversarial examples with small perturbation, polynomial time adversaries are unable to do so   unless they can break standard cryptographic hardness assumptions.  Our results, therefore, indicate that perhaps a similar approach to cryptography (relying on computational hardness) holds promise for achieving computationally robust machine learning. 
On the reverse directions, we also show that the existence of such learning task in which computational robustness beats information theoretic robustness requires computational hardness by implying (average-case) hardness of $\NP$.
\end{abstract}



\section{Introduction}

Designing classifiers that are robust to small perturbations to test instances has emerged as a challenging task in machine learning. The goal of robust learning is to design classifiers $h$ that still correctly predicts the true label even if the input $x$ is  perturbed minimally to a ``close'' instance $x'$. In fact, it was shown \citep{Szegedy:intriguing,Evasion:TestTime,Adversarial::Harnessing} that many learning algorithms, and in particular DNNs, are highly vulnerable to such small perturbations and thus adversarial examples can be successfully
found. Since then, the machine learning community has been actively engaged  to address this problem with many new defenses \citep{papernot2016distillation,madry2017towards,biggio2018wild} and novel and 
powerful attacks \citep{carlini2017towards,athalye2018obfuscated}.

\paragraph{Do adversarial examples  always \emph{exist}?} This state of affairs suggest that perhaps the existence of adversarial example is due to fundamental reasons that might be  inevitable. A sequence of work \citep{gilmer2018adversarial,fawzi2018adversarial,diochnos2018adversarial,mahloujifar2018curse,shafahi2018adversarial,dohmatob2018limitations} show that for natural theoretical distributions (e.g., isotropic  Gaussian of dimension $n$) and natural metrics over them (e.g., $\ell_0, \ell_1$ or $\ell_2$), adversarial examples are inevitable. Namely, the concentration of measure phenomenon \citep{ledoux2001concentration,milman1986asymptotic} in such metric probability spaces imply that small perturbations are enough to map almost all the instances $x$ into a close $x'$ that is misclassified. This line of work, however, does not yet say anything about ``natural'' distributions of interest such as images or voice, as the precise nature of such distributions are yet to be understood.

\paragraph{Can lessons from cryptography help?} Given the pessimistic state of affairs,   researchers have asked if we could use lessons from cryptography to make progress on this problem \citep{MadryTalk,ShafiTalk,mahloujifar2018can}. Indeed, numerous cryptographic tasks (e.g. encryption of long messages) can only be realized against attackers that are computationally bounded. In particular, we know that all encryption methods that use a short key to encrypt much longer messages are insecure against computationally unbounded adversaries. However, when restricted to computationally bounded adversaries this task becomes feasible and suffices for numerous settings. This insight has been extremely influential in cryptography. Nonetheless, despite attempts to build on this insight in the learning setting, we have virtually no evidence on whether this approach is promising. Thus, in this work we study the following question:

\begin{quote}
\emph{Could we  hope to leverage computational hardness for the benefit of adversarially robust learning by rendering successful attacks computationally infeasible?}
\end{quote}

Taking a step in realizing this vision, we provide formal definitions for \emph{computational} variants of robust learning. Following the cryptographic literature, we provide a game based definition of computationally robust learning. Very roughly, a game-based definition consists of two entities: a \emph{challenger} and an \emph{attacker}, that interact with each other. In our case, as the first step the challenger generates independent samples from the distribution at hand, use those samples to train a learning algorithm, and obtain a hypothesis $h$. Additionally, the challenger samples a fresh challenge sample $x$ from the underlying distribution. Next, the challenger provides the attacker with oracle access to $h(\cdot)$ and $x$. At the end of this game, the attacker outputs a value $x'$ to the challenger. The attacker declares this execution as a ``win'' if $x'$ is obtained as a small perturbation of $x$ and leads to a misclassification. We say that the learning is computationally robust as long as no attacker from a class of adversaries can ``win'' the above game with a probability much better than some base value.   
(See Definition \ref{def:game}.)
This definition is very general and it implies various notions of security by restricting to various classes of attackers. While we focus on polynomially bounded attackers in this paper, we remark that one may also naturally consider other natural classes of attackers based on the setting of interest (e.g. an attacker that can only
modify certain part of the image).
\paragraph{What if adversarial examples are actually easy to find?} 
\citet{mahloujifar2019can}  studied this question, and showed that as long as the input instances come from a product distribution, and if the distances are measured in Hamming distance, adversarial examples with sublinear perturbations can be found in polynomial time. This result, however, did not say anything about other
distributions or  metrics such as $\ell_\infty$. Thus, it was left open whether computational hardness could be leveraged in any learning problem to guarantee its robustness.

\subsection{Our Results}
\paragraph{From computational hardness to computational robustness.} In this work, we show that computational hardness can indeed  be leveraged to help robustness. In particular, we present a learning problem $\problem$ that has a classifier $h_\problem$ that is \emph{only} computationally robust. In fact, let $\sfQ$ be \emph{any} learning problem   that has a classifier   with  ``small'' risk $\alpha$, but that adversarial examples  \emph{exist}  for classifier $h_\sfQ$ with higher probability $\beta \gg \alpha$ under the $\ell_0$ norm (e.g., $\sfQ$ could be \emph{any} of the well-studied problems in the literature with a vulnerable classifier $h_\cQ$ under norm $\ell_0$). Then, we show that there is a ``related'' problem $\problem$ and a related classifier $h_\problem$ that has \emph{computational} risk (i.e., risk in the presence of  \emph{computationally bounded} tampering adversaries) at most $\alpha$, but the risk of $h_\problem$ will go up all the way to  $\approx \beta$ if the tampering attackers are allowed to be computationally unbounded. Namely, computationally bounded adversaries have a much smaller chance  of finding adversarial examples of small perturbations for $h_\problem$ than computationally unbounded attackers do. (See Theorem \ref{thm:comprobhyp}.)

The computational robustness of the above construction relies on allowing the hypothesis to sometimes ``detect'' tampering and output a special symbol $\star$. The goal of the attacker is to make the hypothesis output a wrong label and \emph{not} get detected. Therefore, we have proved, along the way, that allowing tamper detection can also be useful for robustness. Allowing tamper detection, however, is not always an option. For example a real-time decision making classifier (e.g., classifying a traffic sign) that \emph{has to} output a label, even if it detects that something might be suspicious about the input image. We prove that even in this case, there is a learning problem $\problem$ with binary labels and a classifier $h$ for $\problem$ such that computational risk of $h$ is almost zero, while its information theoretic risk is $\approx 1/2$, which makes classifiers' decisions under attack meaningless. (See  Theorem \ref{thm:comprobhyp3}).

\sanjam{Do you think that we might be able to sell the fact that it might actually be possible to reasonably modify normal learning problems Q to $P_Q$ such that this holds. For example, every time a camera takes a picture it can generate actually publish the signed version of that picture. Or if you expect certain signs to read by a ML algorithm then you attach a signature along with it. The fact that it can be for any public key sampled at that moment but not a fixed public key is not that obvious.}\Snote{I think we can sell this fact. For instance if we have a cloud based face recognition and the photos taken by camera are signed by an specific signing key that is secured in hardware? This reduces the problem to physical attacks e.g adversary wearing a mask.}

\paragraph{Extension: existence of learning \emph{problems} that are computationally  robust.} Our result above applies to certain classifiers that ``separate'' the power of computationally bounded vs. that of computationally unbounded attackers. Doing so, however, does not rule out the possibility of finding information theoretically robust classifiers for the \emph{same} problem.
So, a natural question is: can we extend our result to show the existence of learning tasks for which \emph{any} classifier is vulnerable to unbounded attackers, while computationally robust classifiers for that task exist? At first, it might look like an impossible task, in ``natural'' settings, in which the ground truth function $c$ itself is robust under the allowed amount of perturbations. (For example, in case of image classification, Human is the robust ground truth). Therefore, we cannot simply extend our result in this setting to rule out the existence of robust classifiers, since they might simply exist (unless one puts a limits on the complexity of the learned model, to exclude the ground truth function as a possible hypothesis).

However, one can still formulate the question above in a meaningful way as follows: Can we have a learning task  for which any \emph{polynomial time} learning algorithm  (with polynomial sample complexity) is forced to produce (with high probability) hypotheses with low robustness against unbounded attacks? Indeed, in this work we also answer this question affirmatively, as a corollary to our main result, by also relying on recent results proved in recent exciting works of \citep{bubeck2018adversarial1,bubeck2018adversarial2,degwekar2019computational}.

In summary, our work provides credence that perhaps  restricting attacks to computationally bounded adversaries holds promise for achieving computationally robust machine learning that relies on computational hardness assumptions as is currently done in cryptography.

\paragraph{From computational robustness  back to computational hardness.} Our first result shows that computational hardness can be leveraged in some cases to obtain nontrivial computational robustness that beats information theoretic robustness. But how about the reverse direction; are computational hardness assumptions necessary for this goal? We also prove such reverse direction and show that nontrivial computational robustness implies computationally hard problems in $\NP$. In particular, we show that a non-negligible gap between the success probability of computationally bounded vs. that of unbounded adversaries in attacking the robustness of classifiers implies strong average-case hard distributions for class $\NP$. Namely, we prove  that if the distribution $D$ of the instances in learning task is efficiently samplable, and if a classifier $h$ for this problem has computational robustness $\alpha$, information theoretic robustness $\beta$, and $\alpha < \beta$, then  one can efficiently sample from a distribution $S$ that generates Boolean formulas $\phi \gets S$ that are satisfiable with overwhelming probability, yet no efficient algorithm can find the satisfying assignments of $\phi \gets S$ with a non-negligible probability. (See Theorem \ref{thm:NP} for the formal statement.) 
\SomeshNote{Can these computationally hard problems in NP be used to construct cryptographic primitives? If so, we should say so.} \MohNote{Unfortunately no.}

\paragraph{What world do we live in?} As explained above, our main question is whether adversarial examples could be prevented by relying on computational limitations of the adversary. In fact, even if adversarial examples exist for a classifier, we might be living in either of two completely different worlds. One is a world in which computationally unbounded adversaries can find adversarial examples (almost) whenever they exist and they would be as powerful as information-theoretic adversaries. Another world is one in which machine learning could leverage computational hardness. Our work suggests that computational hardness can potentially help robustness for certain learning problems; thus, we are living in the better world. Whether or not we can achieve \emph{computational} robustness for \emph{practical} problems (such as image classification) that beats  their information-theoretic robustness remains an intriguing open question. A related line of work~\citep{bubeck2018adversarial1,bubeck2018adversarial2,degwekar2019computational} studied other ``worlds'' that we might be living in, and studied whether adversarial examples are due to the computational hardness of \emph{learning} robust classifiers. They designed learning problems demonstrating that in some worlds, robust classifiers might exist, while they are hard to be obtained efficiently. We note however, that the goal of those works and our work are quite different. They deal with how computational constraints might be an issue and \emph{prevent} the learner from reaching its goal, while our focus is on how such constraints on adversaries can \emph{help} us achieve robustness guarantees that are not achievable information theoretically.

\paragraph{What does our result say about robustifying other natural learning tasks?} Our results only show the \emph{existence} of a learning task for which computational robustness is very meaningful. So, one might argue that this is an ad hoc phenomenon that might not have an impact on other practical problems (such as image classification). However, we emphasize that prior to our work, there was \emph{no} provable evidence that computational hardness can play any positive role in robust learning. Indeed, our results also shed light on how computational robustness can potentially be applied to other, perhaps more natural learning tasks.
The reason is that the space of all possible ways to tamper a high dimensional vector is \emph{exponentially} large. Lessons from cryptography, and the construction of our learning task proving our main result, suggest that, in such cases, there is potentially a huge gap between the power of computationally bounded vs. unbounded search algorithms. On the other hand, there are methods proposed by researchers that \emph{seem} to resist attacks that try to find adversarial examples \citep{madry2017towards}, while the certified robustness literature is all focused on modeling the adversary as a computationally unbounded entity who can find adversarial examples within a certain distance, so long as they exist \citep{raghunathan2018certified,wong2018provable,sinha2018certifiable,wong2018scaling}. Our result shows that, perhaps we shall start to consider \emph{computational} variants of certification methods that focus on computationally bounded adversaries, as by doing so we might be able to prove better robustness bounds for methods that are designed {already}.

\paragraph{Other related work.} In another line of work~\citep{raghunathan2018certified,wong2018provable,sinha2018certifiable,wong2018scaling}   the notion of \emph{certifiable} robustness was developed to prove robustness for   individual test instances. More formally, they aim at providing
robustness certificates with bounds $\eps_x$ along with a decision $h(x)$ made on a test instance $x$, with the guarantee that any $x'$ at distance at most $\eps_x$ from $x$ is correctly classified. However, these guarantees, so far, are not strong enough to rule out attacks completely, as larger magnitudes of perturbation (than the levels certified) still can fool the classifiers while the instances look the same to the human.
\SomeshNote{Can we cite something for this?} \MohNote{citation is the same set of papers; I wrote the sentence again to minimize the attack surface :)}


\subsubsection{Techniques}
We prove our main result about the possibility of computationally robust classifiers (Theorem \ref{thm:comprobhyp}) by ``wrapping'' an arbitrary learning problem $\sfQ$ with a vulnerable classifier  by adding \emph{computational} certification based on cryptographic digital signatures to test instances. A digital signature scheme (see Definition \ref{def:sign}) operates based on  two generated keys $(\vk,\sk)$, where $\sk$ is private and is used for \emph{signing} messages, and $\vk$ is public and is used for \emph{verifying} signatures. Such schemes come with the guarantee that a computationally bounded adversary with the knowledge of $\vk$ cannot sign new messages on its own, even if it is given signatures on some previous messages. Digital signature schemes can be constructed based on the assumption that one-way functions exist.\footnote{Here, we need signature schemes with ``short'' signatures of poly-logarithmic length over the security parameter. They could be constructed based on  exponentially hard one-way functions \citep{rompel1990one} by picking the security parameter sub-exponentially smaller that usual and using universal one-way hash functions to hash the message to poly-logarithmic length..} Below we describe the ideas behind this result in two steps.

{\em \underline{Initial Attempt}.} Suppose $D_\sfQ$ is  the distribution over $\X \times \Y$ of a learning problem $\sfQ$ with input space $\X$ and label space $\Y$. Suppose $D_\sfQ$ had a hypothesis $h_\sfQ$ that can predict correct labels reasonably well, $\Pr_{(x,y)\gets D_\sfQ}[h(x)\neq y] \leq \alpha$. Suppose, at the same time, that a (perhaps computationally unbounded) adversary $\adv$ can perturb test instances like $x$ into a close adversarial example $x'$ that is now likely to be misclassified by $h_\sfQ$, 
\begin{equation*}
    \Pr_{(x,y)\gets D_\sfQ}[h(x')\neq y; x'=\adv(x)] \geq \beta \gg \alpha.
\end{equation*} 
Now we describe a related problem $\problem$, its distribution  of examples $D_\problem$,  and a classifier $h_\problem$ for $\problem$. To sample an example from $D_\problem$, we first sample $(x,y) \gets D_\sfQ$ and then modify $x$ to $\ol{x}=(x,\sigma_x)$ by attaching a \emph{short} signature $\sigma_x=\Sign(\sk,x)$ to $x$. The label $y$ of $\ol{x}$ remains the same as that of $x$. Note that $\sk$ will be kept secret to the sampling algorithm of $D_\problem$. The new classifier $h_\problem$ will rely on the public parameter $\vk$ that is available to it. Given an input $\ol{x}=(x,\sigma_x)$, $h_\problem$ first checks its integrity by verifying that the given signature $\sigma_x$ is valid for $x$. If the signature verification does not pass, $h_\problem$ rejects the input as adversarial without outputting a label, but if this test passes,  $h_\problem$ outputs $h_\sfQ(x)$. 

To successfully find an adversarial example $\ol{x}'$ for $h_\problem$   through a small perturbation of $\ol{x}=(x,\sigma)$ sampled as $(\ol{x},y)\gets D_\problem$ , an adversary $\adv$ can pursue  either of the following strategies. \textbf{(I)} One strategy is that $\adv$ tries to find a new signature $\sigma' \neq \sigma_x$ for the same $x$, which will constitute as a sufficiently small perturbation as the signature is short. Doing so, however, is not considered a successful attack, as the label of $\ol{x}'$ remains the same as that of the true label of the untampered point $\ol{x}$. \textbf{(II)} Another strategy is to perturb the  $x$ part of $\ol{x}$ into a close instance $x'$ and then trying to find a correct signature $\sigma'$ for it, and outputting $\ol{x}'=(x',\sigma')$. Doing so would be a successful attack, because the signature is short, and thus $\ol{x}'$ would indeed be a close instance to $\ol{x}$. However, doing this is computationally infeasible, due to the very security definition of the signature scheme.  Note that $(x'\sigma')$ is a forgery for the signature scheme, which a computationally
bounded adversary cannot construct because of the security of the underlying signature scheme. This means that the \emph{computational} risk of $h_\problem$ would remain at most $\alpha$.

We now observe that information theoretic (i.e., computationally unbounded) attackers can succeed in finding adversarial examples for $h_\problem$ with probability at least $\beta \gg \alpha$. In particular, such attacks can first find an adversarial example $x'$ for $x$ (which is possible with probability $\beta$ over the sampled $x$), construct a signature $\sigma'$ for $x'$, and then output $(x',\sigma')$. Recall  that an unbounded adversary can 
construct a signature $\sigma'$ for $x'$ using exhaustive search.

{\em \underline{Actual construction}.} One main issue with the above construction is that it needs to make $\vk$ publicly available, as a public parameter to the hypothesis (after it is sampled as part of the description of the distribution $D_\problem$). Note that it is computationally hard to  construct the hypothesis described above without knowing $vk$. The problem with revealing $vk$ to the learner is that the distribution of examples should come with some extra information other than samples. However, in the classical definition of a learning problem, the learner only has access to samples from the distribution. In fact, if we were allowed to pass some extra information to the learner, we could pass the description of a robust classifier (e.g. the ground truth) and the learning task becomes trivial. The other issue is that the distribution $D_\problem$ is not  \emph{publicly samplable} in polynomial time, because to get a sample from $D_\problem$ one needs to use the signing key $\sk$, but that key is kept secret. We resolve these two issues with two more ideas. 
The first idea is that, instead of generating \emph{one} pair of keys $(\vk,\sk)$ for $D_\problem$ and keeping $\sk_D$ secret,  we can generate a fresh pair of keys $(\vk_x,\sk_x)$ every time that we sample $(x,y) \gets D_\sfQ$ and attach $\vk_x$ also to the actual instance $\ol{x}=(x,\sigma_x,\vk_k)$. The modified hypothesis $h_\problem$ also uses this key and verifies $(x,\sigma_x)$ using $\vk_x$. This way, the distribution $D_\problem$ is publicly samplable, and moreover, there is no need for making $\vk$ available as a public parameter. However, this change of the distribution $D_\problem$ introduces a new possible way to attack the scheme and to find adversarial examples.  In particular, now the adversary can try to perturb $\vk_x$ into a close string $\vk'$ for which it knows a corresponding signing key $\sk'$, and then use $\sk'$ to sign an adversarial example $x'$  for $x$ and output $(x',\sigma',\vk')$. However, to make this attack impossible for the attacker under small perturbations of instances, we use error correction codes and employ an encoding $[\vk_x]$ of the verification key (instead of $\vk_x$) that needs too much change before one can fool a decoder to decode to any other $\vk'\neq \vk_x$. But as long as the adversary cannot change $\vk_x$, the adversary cannot attack the robustness computationally. (See Construction~\ref{const:Learningfamily}.)

To analyze the construction above (see Theorem \ref{thm:comprobhyp}), we note that the computationally bounded adversary would need to change $\Omega(|x|)$ number of bits in $(x,\sigma, [\vk])$ to get $(x',\sigma',[\vk'])$ where $x\neq x'$. This is because because the encoded $[\vk]$ would need $\Omega(|x|)$ number of perturbations to change the encoded $\vk$, and if $\vk$ remains the same it is hard computationally to find a valid signature. On the other hand, a computationally unbounded adversary can focus on perturbing $x$ into $x'$ and then forge a short signatures for it, which could be as small as $\poly(\log (|x|))$ perturbations.

\paragraph{Extension to problems, rather than specific classifiers for them.} Note that the construction above could be wrapped around any learning problem. In particular, we can pick an original problem that is not (information theoretically) robustly learnable in polynomial time. These problems, which we call them \hardrobust are studied recently in \citep{bubeck2018adversarial1} and \citep{degwekar2019computational} where they construct such \hardrobust problems to show the effect of computational limitation in robust learning (See Definition \ref{def:hardrobust} and \ref{thm:hardrobust}). Here, using their construction as the original learning problem, and wrapping it with our construction, we can strengthen our result and construct a learning problem that is not robustly learnable by \emph{any} polynomial time learning algorithm, yet it has a classifier that is computationally robust. See Corollary \ref{cor:comprob_any} for more details.

\paragraph{Computational robustness without tamper detection.} The computational robustness of the constructed classifier relies on sometimes detecting tampering attacks and not outputting a label. We give an alternative construction  for a setting that the classifier \emph{always} has to output a label. We again use digital signatures and error correction codes as the main ingredient of our construction but in a different way. The main difference is that we have to repeat the signature multiple times to prevent the adversary from changing all of the signatures. The caveat of this construction is that it is no longer a wrapper around an arbitrary learning problem. See Construction \ref{const:Learningfamily3} for more details.


\section{Defining Computational Risk and Computationally Robust Learning}
\paragraph{Notation.} We  use calligraphic letters (e.g., $\cX$) for sets and capital non-calligraphic letters (e.g., $D$) for distributions. 
By $d \gets D$ we denote sampling $d$  from $D$. 
For a randomized algorithm $R(\cdot)$,  $y \gets R(x)$ denotes the randomized execution of $R$ on input $x$ outputting $y$.  
 A classification problem $\problem=(\X,\Y,D,\H)$ is specified by the following components: 
set $\X$ is the set  of possible \emph{instances}, \Y is the set of possible \emph{labels},
$D \in \D$ is a \emph{joint} distribution over $\X \times \Y$, and $\H$
is the space of hypothesis.  For simplicity  we work with problems that have a single distribution $D$ (e.g., $D$ is the distribution of labeled images from a data set like MNIST or CIFAR-10). A learner $L$ for problem $P$ is an algorithm that takes a dataset $\cS \gets D^m$ as input and outputs a hypothesis $h\in \cH$.
We did not state the loss function explicitly, as we work with classification problems and use the zero-one loss by default.
For a learning problem $\problem=(\X,\Y,D,\H)$, the \emph{risk} or \emph{error} of a hypothesis $h \in \H$ is $\Risk_\problem(h) = \Pr_{(x,y) \gets D}[h(x) \neq y]$. We are usually interested in learning problems $\problem=(\X,\Y,D,\H)$ with a specific metric $\metric$ defined over $\X$ for the purpose of defining adversarial perturbations of bounded magnitude controlled by $\metric$. In that case, we might simply write $\problem=(\X,\Y,D,\H)$, but  $\metric$ is implicitly defined over $\X$.  Finally, for a metric $\metric$ over $\X$, we let $\metric_b(x)=\set{x' \mid \metric(x,x') \leq b}$ be the ball of radius $b$ centered at $x$ under the metric $\metric$. By default, we work with Hamming distance $\HD(x,x') = |\set{i \colon x_i \neq x'_i}|$,
 but our definitions   can be adapted to any other metrics.
We usually work with \emph{families} of problems $\problem_n$ where $n$ determines the length of $x \in \X_n$ (and thus input lengths of $h \in \H_n,c\in \C_n, \metric_n$). 
We sometimes use a special notation $\Pr[{x \gets X}; E(x)]$ to define $\Pr_{x\gets X}[E(x)]$ that is the probability of and event $E$ over a random variable $X$. We also might use a combination of multiple random variables, for examples $\Pr[x\gets X; y\gets Y; E(x,y)]$ denotes the same thing as $\Pr_{\substack{ x\gets X, y\gets Y}}[E(x,y)]$. Order of sampling of $X$ and $Y$ matters  $Y$ might depend on $X$.

\paragraph{Allowing tamper detection.} In this work we expand the standard notion of hypotheses and allow $h \in \H$ to output a special symbol $\star$ as well (without adding $\star$ to $\Y$), namely we  have $h \colon \X \To \Y \cup \set{\star}$. This symbol can be used to denote ``out of distribution'' points, or any form of tampering. In natural scenarios, $h(x)\neq \star$ when $x$ is not an adversarially tampered instance. However, we allow this symbol to be output by $h$ even in no-attack settings as long as its probability is small enough.

We follow the tradition of game-based security definitions in cryptography \citep{naor2003cryptographic,shoup2004sequences,goldwasser2016cryptographic,rogaway2018simplifying}. Games are the most common way that security is defined in cryptography. These games are defined between
 a challenger $\Chal$ and an adversary $\adv$. Consider the case of a signature scheme. In this case 
 the challenger $\Chal$ is a signature scheme $\Pi$ and an adversary $\adv$ is given oracle access to the signing functionality
 (i.e. adversary can give a message $m_i$ to the oracle and obtains the corresponding signature $\sigma_i$).
 Adversary $\adv$ wins the game if he can provide a valid signature on a message that was not queried
 to the oracle. The security of the signature scheme is then defined informally as follows: any
 probabilistic polynomial time/size adversary  $\adv$ can win the game by probability that is bounded
 by a negligible $n^{-\omega(1)}$ function on the security parameter. We describe a security game for  tampering adversaries with bounded tampering budget in $\HD$, 
 but the definition is more general and can be used for other adversary classes.

\sanjam{add some background on game based notions} \Mnote{added citations to the beginning of paragraph above}

\begin{definition}[Security game of adversarially robust learning] \label{def:game}
Let $\problem_n=(\X_n,\Y_n,D_n,\H_n)$ be a  classification problem  where the components are parameterized by $n$. Let $L$ be a learning algorithm  with sample complexity $m=m(n)$ for $\problem_n$. Consider the following game between a challenger $\Chal$, and an adversary $\adv$ with tampering budget $b=b(n)$.
\begin{enumerate}
    \item $\Chal$ samples $m$ i.i.d. examples $\cS \gets D_n^m$ and  gets hypothesis $h \gets L(\cS)$ where $h \in \H_n$.
    \item $\Chal$ then samples a test example $(x,y) \gets D_n$ and sends  $(x,y)$ to the adversary $\adv$.
    \item Having  oracle (gates, in case of circuits) access  to hypothesis $h$ and a sampler for $D_n$, the adversary obtains the adversarial instance $x'\gets \adv^{h(\cdot),D_n}(x)$ and  outputs  $x'$.
\end{enumerate}
{\em Winning conditions:} In case $x=x'$, the adversary $\adv$ wins  if $h(x)\neq y$,\footnote{Note that, if $h(x)\neq y$, without loss of generality, the adversary  $\adv$ can output $x'=x$} and in case $x \neq x'$, the adversary wins if all the following hold:
\begin{inparaenum}[(1)]
    \item $\HD(x,x') \leq b$,
    \item $h(x') \neq y$, and
    \item $h(x') \neq \star$.
\end{inparaenum}
\end{definition}

\paragraph{Why separating winning conditions for $x=x'$ from $x\neq x'$?} One might wonder why we separate the winning condition for the two cases of $x=x'$ and $x\neq x'$. The reason is that $\star$ is  supposed to capture tamper detection. So, if the adversary does not change  $x$ and the hypothesis outputs $h(x)=\star$, this is an error, and thus should contribute to the risk. More formally, when we evaluate risk, we have $\Risk_{\problem}(h) = \Pr_{(x,y) \gets D}[h(x) \neq y]$, which implicitly means that outputting $\star$  contributes to the risk.
 However,  if adversary's perturbs to $x'\neq x$ leads to  $h(x')=\star$, it means the adversary has \emph{not} succeeded in its attack, because the tampering is detected.
In fact, if we simply require the other 3 conditions to let adversary win, the notion of ``adversarial risk'' (see Definition~\ref{def:advRiskComp}) might be even less than the normal risk, which is counter intuitive.




\paragraph{Alternative definitions of winning for the adversary.} The winning condition for the adversary could be defined in other ways as well. In our Definition~\ref{def:game}, the adversary wins if $\metric(x,x') \leq b$ and $h(x') \neq y$. This condition is inspired by the notion of corrupted input \citep{feige2015learning},  is extended to metric spaces  in \citep{madry2017towards}, and is used in and many subsequent works. An alternative definition for adversary's goal, formalized in \citep{diochnos2018adversarial} and used in \citep{gilmer2018adversarial,diochnos2018adversarial,bubeck2018adversarial2,degwekar2019computational} requires $h(x')$ to be different from the true label of $x'$ (rather than $x$). This condition requires the misclassification of $x'$, and thus, $x'$ would belong to the ``error-region'' of $h$. Namely, if we let $c(x)=y$ be the ground truth function, the error-region security game requires $h(x') \neq c(x')$. Another stronger definition of adversarial risk is given by \cite{suggala2018adversarial} in which the requirement condition requires both conditions: (1) the ground truth should not change $c(x)=c(x')$, and that (2) $x'$ is misclassified. For natural distributions like images or voice, where the ground truth  is robust to small perturbations, all these three  definitions for adversary's winning are equivalent.  

\paragraph{Stronger attack models.} In the attack model of Definition \ref{def:game}, we only provided the label $y $ of $x$ to the adversary and also give her the sample oracle from $D_n$. A stronger attacker can have access to the ``concept'' function $c(x)$ which is sampled from the distribution of $y$ given $x$ (according to $D_n$). This concept oracle might not be efficiently computable, even in scenarios that $D_n$ is efficiently samplable. In fact, even if $D_n$ is not efficiently samplable, just having access to a large enough pool of i.i.d. sampled data from $D_n$ is enough to run the experiment of Definition \ref{def:game}. In alternative winning conditions (e.g., the error-region definition) for Definition \ref{def:game} discussed above, it makes more sense to also include the ground truth concept oracle $c(\cdot)$ given as oracle to the adversary, as the adversary needs to achieve $h(x') \neq c(x')$. Another way to strengthen the power of adversary is to give him non-black-box access to the components of the game (see~\cite{papernot2017practical}). In definition \ref{def:game}, by default, we model adversaries who have black-box access to $h(\cdot),D_n$, but one can define non-black-box (white-box) access  to each of $h(\cdot),D_n$, if they are polynomial size objects. 



\citet{diochnos2018adversarial} focused on bounded perturbation adversaries that are unbounded  in their running time and formalized notions of of adversarial risk for a given hypothesis $h$ with respect to the   $b$-perturbing adversaries.  Using Definition~\ref{def:game}, in Definition \ref{def:advRiskComp},  we retrieve the notions of standard risk, adversarial risk, and its (new) computational variant.

\begin{definition}[Adversarial risk of hypotheses and learners] \label{def:advRiskComp}
Suppose $L$ is a learner for a problem $\problem=(\X,\Y,D,\H)$. For a class of attackers $\cA$ we define 
\begin{equation*}\AdvRisk_{\problem,\cA}(L)=\sup_{\adv \in \cA} \Pr[\adv \text{ wins}]\end{equation*}
where the winning is   in the experiment of Definition~\ref{def:game}.
When the attacker $\adv$ is fixed, we simply write $\AdvRisk_{\problem,\adv}(L)=\AdvRisk_{\problem,\set{\adv}}(L)$.
For a trivial attacker $I$ who outputs $x'=x$, it holds that $\Risk_\problem(L) = \AdvRisk_{\problem,I}(L)$.
When $\cA$ includes attacker that are \emph{only} bounded by $b$ perturbations, we use notation $\AdvRisk_{\problem,b}(L)=\AdvRisk_{\problem,\cA}(L)$, and when the adversary is further restricted to all $s$-size (oracle-aided) circuits, we use notation $\AdvRisk_{\problem,b,s}(L)=\AdvRisk_{\problem, \cA}(L)$.
When $L$ is a learner that outputs a fixed hypothesis $h$,  by substituting $h$ with $L$, we obtain the following similar  notions for $h$, which will be denoted as   $\Risk_{\problem}(h)$,   $\AdvRisk_{\problem,\cA}(h)$,   $\AdvRisk_{\problem, b}(h)$, and   $\AdvRisk_{\problem, b,s}(h)$.

\end{definition}  

\begin{definition}[Computationally robust learners and hypotheses] \label{def:compRob}
Let $\problem_n=(\X_n,\Y_n,D_n,\cH_n)$ be a family of classification  parameterized by $n$. We say that a learning algorithm $L$ is a \emph{computationally robust} learner   with risk  at most $R=R(n)$ against $b=b(n)$-perturbing adversaries,  if  for any polynomial $s=s(n)$, there is a negligible function $\negl(n)=n^{-\omega(1)}$ such that
\begin{equation*}\AdvRisk_{\problem_n, b,s}(L) \leq R(n) 
 \negl(n)
.\end{equation*}

Note that the size of circuit used by the adversary controls its computational power and that is why we are enforcing it to be a polynomial. Again, when $L$ is a learner that outputs a fixed hypothesis $h_n$ for each $n$, we say that the family $h_n$ is a computationally robust hypothesis with risk at most $R=R(n)$ against $b=b(n)$-perturbing adversaries, if $L$ is so. In both cases, we might simply  say that $L$  (or $h$) has \emph{computational risk} at most $R(n)$.
\end{definition} 

\begin{remark}
We remark that, in the definition above, one can opt to work with concrete bounds and  a version that drops the negligible probability $\negl$ on the right hand side of the equation and ask for the upper bound to be simply stated as $\AdvRisk_{\problem_n, b,s}(L) \leq R(n) $. Doing so, however, is a matter of style. In this work, we opt to work with the above definition, as the negligible probability usually comes up in computational reductions, and hence it simplifies the statement of our theorems, but both forms of the definition of computational risk are equally valid.
\end{remark}

\paragraph{PAC learning under computationally bounded tampering adversaries.} Recently, several works studied generalization under adversarial perturbations from a theoretical perspective \citep{bubeck2018adversarial,cullina2018pac,feige2018robust,attias2018improved,khim2018adversarial,yin2018rademacher,montasser2019vc,diochnos2019lower}, and hence they implicitly or explicitly revisited the ``probably approximately corect'' (PAC) learning framework of \cite{Valiant::PAC_Book} under adversarial perturbations.
Here we comment that, one can derive variants of those definitions for \emph{computationally bounded} attackers, by limiting their adversaries as done in  our Definition \ref{def:compRob}. In particular, we call a learner $L$ an $(\eps,\delta)$ PAC learner for a problem $\problem$ and computationally bounded $b$-perturbing adversaries, if with probability $1-\delta$, $L$ outputs a hypothesis $h$ that has computational risk at most $\eps$.


Bellow we formally define the notion of \hardrobust learning problems which captures the inherent vulnerability of a learning problem to adversarial attacks due to computational limitations of the learning algorithm. This definition are implicit in works of \citep{degwekar2019computational,bubeck2018adversarial1}. In Section \ref{sec:SepWith}, we use these \hardrobust problems to construct learning problems that are inherently non-robust in presence of computationally unbounded adversaries but have robust classifiers against  computationally bounded adversaries. 

\begin{definition}[\Hardrobust learning problems] \label{def:hardrobust}
A learning problem $\problem_n=(\X_n,\Y_n,D_n,\H_n)$ is \hardrobust w.r.t budget $b(n)$ if for any learning algorithm $L$ that runs in $\poly(n)$ we have
$$\AdvRisk_{\problem_n, b(L)} \geq 1 -\negl(n).$$
\end{definition}

\begin{theorem}[\cite{degwekar2019computational, bubeck2018adversarial1}]\label{thm:hardrobust}
There exist a Learning problem $\problem_n=(\X_n,\Y_n,D_n,\H_n)$ and a sub-linear budget $b(n)$ such that $\problem_n$ is \hardrobust w.r.t $b$
unless one-way functions do not exist.  (See appendix for the definition of one-way functions)
\end{theorem}

\paragraph{Discussion on falsifiability of  computational robustness.} If the learner $L$ is polynomial time, and that the distribution $D_n$ is samplable in polynomial time (e.g., by sampling $y$ first and then using a generative model to generate $x$ for $y$), then the  the computational robustness of learners as defined based on Definitions \ref{def:compRob} and \ref{def:game} is a ``falsifiable'' notion of security as defined by  \citet{naor2003cryptographic}. Namely, if an adversary claims that it can break the computational robustness of the learner $L$, it can prove so in polynomial time by participating in a challenge-response game and winning in this game with a noticeable (non-negligible) probability more than $R(n)$. This feature is due to the crucial property of the challenger in Definition \ref{def:game} that is a polynomial time algorithm itself, and thus can be run efficiently. Not all security games have efficient challengers (e.g., see \citet{pandey2008adaptive}).

 \section{From Computational Hardness to Computational Robustness} \label{sec:SepWith}

In this section, we will first prove our main result that shows the existence of a learning problem with classifiers that are only computationally robust. We first prove our result by starting from \emph{any} hypothesis that is vulnerable to adversarial examples; e.g., this could be any of the numerous algorithms shown to be susceptible to adversarial perturbations. Our constructions use error correction codes and cryptographic signatures. For definitions of these notions refer to section \ref{sec:tools}.

\subsection{Computational Robustness with Tamper Detection}
Our first construction uses hypothesis with tamper detection (i.e, output $\star$ capability). Our construction is based on cryptographic signature schemes with short (polylogarithmic) signatures.


\begin{construction}[Computationally robust problems relaying on tamper detection wrappers]\label{const:Learningfamily}
Let $\mathsf{Q}\xspace=(\set{0,1}^d, \cY, D , \cH)$ be a learning problem and $h\in \cH$ a classifier for $\sfQ$ such that $\Risk_{\sfQ}(h) = \alpha$.
We construct a family of learning problems $\problem_n$ (based on the fixed problem $\sfQ$) with a family of classifiers $h_n$. In our construction we use signature scheme $(\KGen,\Sign,\verify)$ for which the bit-length of $\vk$ is $\lambda$ and the bit-length of signature is $\ell(\lambda)=\polylog(\lambda)$
\footnote{Such signatures exist assuming exponentially hard one-way functions \citep{rompel1990one}.}.
We also use an error correction code $(\encode, \decode)$ with code rate $\mathsf{cr}=\Omega(1)$ and error rate $\mathsf{er}=\Omega(1)$.   
\begin{enumerate}
    \item The space of instances for $\problem_n$ is $\cX_n = \set{0,1}^{n + d+\ell(n)}$.
    \item The set of labels is $\cY_n = \cY$. 
    \item The distribution $D_n$ is defined by the following process: first sample $(x,y)\gets D$, $(\sk, \vk)\gets \KGen(1^{n\cdot \mathsf{cr}})$, $\sigma\gets \Sign(\sk,x)$, then encode $[\vk] =\encode(\vk)$ and output~$\left((x,\sigma,[\vk]),y\right)$.
    \item The classifier $h_n \colon \cX_n \to \cY_n$ is defined as
    \begin{align*}
        h_n(x,\sigma,[\vk]) = \begin{cases} h(x) & \text{if~~} \verify\left(\decode([\vk]), x, \sigma \right),\\
        \star & \textit{otherwise.}
        \end{cases}
    \end{align*}
\end{enumerate}
\end{construction}

\begin{theorem}\label{thm:comprobhyp}
For family $\problem_n$ of Construction \ref{const:Learningfamily}, the family of classifiers $h_n$ is computationally robust with risk at most $\alpha$ against adversaries with budget $\mathsf{er}\cdot n$. (Recall that $\mathsf{er}$ is the error rate of the error correction code.) On the other hand $h_n$ is not robust against information theoretic adversaries of budget $b+\ell(n)$, if $h$ itself is not robust to $b$ perturbations: \begin{equation*}
    \AdvRisk_{\problem_n, b+\ell(n)}(h_n) \geq \AdvRisk_{\mathsf{Q},b}(h).
\end{equation*}
\end{theorem}

Theorem \ref{thm:comprobhyp} means that, the computational robustness of $h_n$ could be as large as $\Omega(n)$ (by choosing a code with constant error correction rate) while its information theoretic adversarial robustness could be as small as $b+\polylog(n) \leq \polylog(n)$ (note that $b$ is a constant here) by choosing a signature scheme with short signatures of poly-logarithmic length.
\\
Before proving Theorem \ref{thm:comprobhyp} we state the following corollary about \hardrobust learning problems.
\begin{corollary}\label{cor:comprob_any}
If the underlying problem $Q$ in Construction \ref{const:Learningfamily} is \hardrobust w.r.t sublinear budget $b(n)$, then for \emph{any} polynomial learning algorithm $L$ for $\problem_n$  we have
$$\AdvRisk_{\problem_n, b+\ell(n)}(L) \geq 1- \negl(n).$$ On the other hand, the family of classifiers $h_n$ for $\problem_n$ is computationally robust with risk at most $\alpha$ against adversaries with linear budget.
\end{corollary}

 The above corollary follows from Theorem \ref{thm:comprobhyp} and definition of \hardrobust learning problems. The significance of this corollary is that it provides an example of a learning problem that could not be learnt robustly with \emph{any} polynomial time algorithm. However, the same problem has a classifier that is robust against computationally bounded adversaries. This construction uses a \hardrobust learning problem that is proven to exist based on cryptographic assumptions \citep{bubeck2018adversarial1,degwekar2019computational}. Now we prove Theorem \ref{thm:comprobhyp}.

\smallskip
\begin{proof} (of Theorem \ref{thm:comprobhyp})
We first prove the following claim about the risk of $h_n$.
\begin{claim}\label{clm:01}
For problem $\problem_n$ we have
\begin{equation*}\Risk_{\problem_n}(h_n) = \Risk_{\sfQ}(h)= \alpha.\end{equation*}
\end{claim}
\begin{proof}
The proof follows from the completeness of the  signature scheme. We have,
\begin{align*}
    \Risk_{\problem_n}(h_n)&= \Pr[\left(\left(x, \sigma, [\vk]\right),y\right)\gets D_n;~ h_n(x,\sigma, [\vk]) \neq y ]\\
    &=\Pr[(x,y) \gets D;~ h(x) \neq y] = \Risk_{\mathsf{Q}}(h).
\end{align*}
\end{proof}
Now we prove the computational robustness of $h_n$.
\begin{claim}\label{clm:02}
For family $\problem_n$, and for any polynomial $s(\cdot)$ there is a negligible function $\negl$ such that for all $n \in \N$ 
\begin{equation*}\AdvRisk_{\problem_n, \mathsf{er}\cdot n,s}(h_n) \leq \alpha + \negl(n).\end{equation*}
\end{claim}
\begin{proof}
Let $\adv_{\set{{n\in \N}}}$ be the family of circuits maximizing the adversarial risk for $h_n$ for all $n \in \N$. We build a sequence of circuits $\adv^1_{\set{{n\in \N}}}$, $\adv^2_{\set{{n\in \N}}}$ such that $\adv^1_n$ and $\adv^2_n$ are of size at most $s(n) + poly(n)$. $\adv^1_n$ just samples a random $(x,y)\gets D$ and outputs $(x,y)$. $\adv^2_n$ gets $x, \sigma$ and  $\vk$, calls $\adv_n$ to get $(x', \sigma', \vk')\gets \adv_n((x,\sigma, [\vk]), y)$ and  outputs $(x',\sigma')$. Note that $\adv^2_n$ can provide all the oracles needed to run $\adv_n$ if the sampler from $D$, $h$ and $c$ are all computable by a circuit of polynomial size. Otherwise, we need to assume that our signature scheme is secure with respect to those oracles and the proof will follow. We have,
\begin{align*}
\AdvRisk_{\problem_n, \mathsf{er}\cdot n, s}(h_n) &= \Pr[((x,\sigma, [\vk]),y)\gets D_n;~ (x',\sigma', \vk')\gets A((x,\sigma, [\vk]),y));~\\
&~~~~~~~~(x',\sigma', \vk') \in \HD_{\mathsf{er}\cdot n}(x,\sigma, [\vk]) \wedge h_n(x',\sigma', \vk')\neq \star \wedge h_n(x',\sigma', \vk') \neq y].
\end{align*}
Note that $(x',\sigma', \vk') \in \HD_{\mathsf{er}\cdot n}(x,\sigma, [\vk])$ implies that $\decode(\vk') = \vk$ based on the error rate of the error correcting code. Also $h_n(x',\sigma', \vk')\neq \star$ implies that $\sigma'$ is a valid signature for $x'$ under verification key $\vk$. Therefore, we have,
\begin{align*}
    &\AdvRisk_{\mathsf{er}\cdot n,s}(h_n) \\
    &\leq \Pr[(\sk,\vk)\gets \KGen(1^n);~ (x,y)\gets \adv_1(1^n);~ \sigma\gets \Sign(\sk,x);~ (x',\sigma') \gets \adv_2(x,\sigma, \vk);~
    \\&~~~~~~~~~~~\verify(\vk,x',\sigma')\wedge h_n(x',\sigma',[\vk]) \neq y]\\
    &\leq \Pr[(\sk,\vk)\gets \KGen(1^n);~ x\gets \adv_1(1^n);~ \sigma\gets \Sign(\sk,x);~ (x',\sigma') \gets \adv_2(x,\sigma, \vk);~
    \\&~~~~~~~~~~~\verify(\vk,x',\sigma') \wedge x'\neq x] + \Risk_{\problem_n}(h_n).
\end{align*}
Thus, by the unforgeability of the one-time signature scheme we have
\begin{equation*}\AdvRisk_{\problem_n, \mathsf{er}\cdot n,s}(h_n)\leq \Risk_{\problem_n}(h_n) + \negl(n),\end{equation*}
which by Claim \ref{clm:01} implies
\begin{equation*}\AdvRisk_{\mathsf{er}\cdot n,s}(h_n)\leq \alpha + \negl(n).\end{equation*}
\end{proof}
Now we show that $h_n$ is not robust against computationally unbounded attacks.
\begin{claim}\label{clm:03}
For family $\problem_n$ and any $n, b \in \N$ we have
\begin{equation*}\AdvRisk_{\problem_n, b+\ell(n)}(h_n) \; \geq \;  \AdvRisk_{\mathsf{Q}, b(h)}.\end{equation*}
\end{claim}

\begin{proof}
For any $((x,\sigma, [\vk]), y)$ define $A(x,\sigma,[\vk]) = (x', \sigma', [\vk])$ where $x'$ is the closes point to $x$ where $h(x)\neq y$ and   $\sigma'$ is a valid signature such that $\verify(\vk,x^*, \sigma')=1$. Based on the fact that the size of signature is $\ell(n)$, we have $\HD(A(x, \sigma, [\vk]) , (x,\sigma,[\vk])) \leq \ell(n) + \HD(x,x').$ Also, it is clear that $h_n(A(x, \sigma, [\vk]))\neq \star$ because $\sigma'$ is a valid signature. Also, $h_n(A(x, \sigma, [\vk])) \neq c_n(A(x, \sigma, [\vk]))$. Therefore we have
\begin{align*}
    &\AdvRisk_{\problem_n, b+\ell(n)}(h_n)\\
    &= \Pr[((x,\sigma, [\vk]),y)\gets D_n; \exists (x',\sigma') \in \HD_{b+\ell(n)}(x,\sigma), h(x') \neq y \wedge h(x') \neq \star \wedge \verify(\vk,\sigma',x') ]\\
    &\geq \Pr[((x,\sigma, [\vk]),y)\gets D_n; \exists x' \in \HD_{b}(x), h(x') \neq y \wedge h(x') \neq \star ]\\
    & = \AdvRisk_{\mathsf{Q}, b}(h).
\end{align*}
\end{proof}
This concludes the proof of Theorem \ref{thm:comprobhyp}.
\end{proof}

\subsection{Computational Robustness without Tamper Detection}

The following theorem shows an alternative construction that is incomparable to Construction \ref{const:Learningfamily}, as it does not use any tamper detection. On the down side, the construction is not defined with respect to an arbitrary (vulnerable) classifier of a natural problem.

\begin{construction}[Computational robustness without tamper detection]\label{const:Learningfamily3}
Let $D$ be a distribution over $\bits^{\mathsf{cr}\cdot n} \times \bits$ with a balanced ``label'' bit: $\Pr_{(x,y)\gets D}[y=0] = 1/2$. \Snote{Do we have a term for consistent distributions? meaning that $(x,0)$ and $(x,1)$ cannot both be in the support set.}
We construct a family of learning problems $\problem_n$ with a family of classifiers $h_n$. In our construction we use a signature scheme $(\KGen,\Sign,\verify)$ for which the bit-length of $\vk$ is $\lambda$ and the bit-length of signature is $\ell(\lambda)=\polylog(\lambda)$ and an error correction code $(\encode, \decode)$ with code rate $\mathsf{cr}=\Omega(1)$ and error rate $\mathsf{er}=\Omega(1)$.
\begin{enumerate}
    \item The space of instances for $\problem_n$ is $\cX_n = \bits^{2n + n\cdot\ell(n)}$.
    \item The set of labels is $\cY_n = \bits$. 
    \item The distribution $D_n$ is defined as follows: first sample $(x,y)\gets D$, then sample $(\sk, \vk)\gets \KGen(1^{n\cdot\mathsf{cr}})$ and compute $[\vk] = \encode(\vk)$. Then compute $[x] = \encode(x)$. If $y=0$ sample a random $\sigma\gets \bits^{\ell(n)}$ that is not a valid signature of $x$ w.r.t $vk$. Then output $(([x],\sigma^n,[\vk]), 0)$. Otherwise compute $\sigma \gets \Sign(\sk, x)$ and output $(([x],\sigma^n, [\vk]), 1)$.
    \item The classifier $h_n \colon \cX_n \to \cY_n$ is defined as
    \begin{align*}
        h_n(x',\sigma_1,\dots,\sigma_n,\vk') = \begin{cases} 1 & \text{if~~} \exists i\in [n]; \verify\left(\decode(\vk'), \decode(x'), \sigma_i \right),\\
        0 & \textit{otherwise.}
        \end{cases}
    \end{align*}
\end{enumerate}
\end{construction}

\begin{theorem}\label{thm:comprobhyp3}
For family $\problem_n$ of Construction \ref{const:Learningfamily3}, the family of classifiers $h_n$ has risk $0$ and is computationally robust with risk at most $0$ against adversaries of budget $\mathsf{er}\cdot n$. On the other hand $h_n$ is not robust against information theoretic adversaries of budget $\ell(n)$: \begin{equation*}
    \AdvRisk_{\problem_n,\ell(n)}(h_n) \geq 1/2.
\end{equation*}
Note that reaching adversarial risk $1/2$ makes the classifier's decisions meaningless as a random coin toss achieves this level of accuracy.
\end{theorem}

\begin{proof}
First it is clear that for problem $\problem_n$ we have
$\Risk_{\problem_n}(h_n) = 0$.
Now we prove the computational robustness of $h_n$.
\begin{claim}\label{clm:04}
For family $\problem_n$, and for any polynomial $s(\cdot)$ there is a negligible function $\negl$ such that for all $n \in \N$ 
\begin{equation*}\AdvRisk_{\problem_n, \mathsf{er}\cdot n,s}(h_n) \leq \negl(n).\end{equation*}
\end{claim}
\begin{proof}
Similar to proof of Claim \ref{clm:02} we prove this based on the security of the signature scheme. Let $\adv_{\set{{n\in \N}}}$ be the family of circuits maximizing the adversarial risk for $h_n$ for all $n \in \N$. We build a sequence of circuits $\adv^1_{\set{{n\in \N}}}$ and $\adv^2_{\set{{n\in \N}}}$ such that $\adv^1_n$ and $\adv^2_n$ are of size at most $s(n) + \poly(n)$. $\adv^1_n$ just asks the signature for $0^{\mathsf{cr}\cdot n}$. $\adv^2_n$ gets $\vk$ and does the following: It first samples $(x,y)\gets D$, computes encodings $[x] = \encode(x)$ and $[\vk]=\encode(\vk)$ and if $y=0$, it samples a random $\sigma\gets\bits^{\ell(n)}$ then calls $\adv_n$ on input $([x],\sigma^n, [\vk])$ to get $(x', (\sigma_1,\dots,\sigma_n), \vk')\gets \adv_n(([x],\sigma^n, [\vk]), y)$. Then it checks all $\sigma_i$'s and if there is any of them that $\verify(\vk,\sigma_i, x)=1$ it outputs $(x,\sigma_i)$, otherwise it aborts and outputs $\bot$. If $y=0$ it aborts and outputs $\bot$. Note that $\adv^2_n$ can provide all the oracles needed to run $\adv_n$ if the sampler from $D$, $h$ and $c$ are all computable by a circuit of polynomial size. Otherwise, we need to assume that our signature scheme is secure with respect to those oracles and the proof will follow. We have,
\begin{align*}
&\AdvRisk_{\problem_n, \mathsf{er}\cdot n, s}(h_n) \\
&= \Pr[(([x],\sigma^n, [\vk]),y)\gets D_n;~ (x',(\sigma_1,\dots,\sigma_n), \vk')\gets A_n(([x],\sigma^n, [\vk]),y));~\\
&~~~~~~~~(x',(\sigma_1,\dots,\sigma_n), \vk') \in \HD_{\mathsf{er}\cdot n}([x],\sigma^n, [\vk]) \wedge h_n(x',(\sigma_1,\dots,\sigma_n), \vk') \neq y].
\end{align*}
Because of the error rate of the error correcting code, $(x',(\sigma_1,\dots,\sigma_n), \vk') \in \HD_{\mathsf{er}\cdot n}(x,\sigma^n, [\vk])$ implies that $\decode(\vk') = \vk$ and $\decode(x') = x$. Also $h_n(x',(\sigma_1,\dots,\sigma_n), \vk')\neq y$ implies that $y=0$. This is because if $y=1$, the adversary has to make all the signatures invalid which is impossible with tampering budget $\mathsf{cr}\cdot n$. Therefore $y$ must be $1$ and one of the signatures in $(\sigma_1,\dots,\sigma_n)$ must pass the verification because the prediction of $h_\lambda$ should be $1$. Therefore we have
\begin{align*}
\AdvRisk_{\problem_n,\mathsf{er}\cdot n, s}(h_n) &\leq \Pr[((x,\sigma^n, [\vk]),y)\gets D_n;~ (x',(\sigma_1,\dots,\sigma_n), \vk')\gets A((x,\sigma, [\vk]),y));~\\
&~~~~~~~~y=0 \wedge \exists i \verify(\vk,\sigma_i,x)]\\
&\leq \Pr[(\sk,\vk)\gets \KGen(1^n);~ 0^{\mathsf{cr}\cdot n}\gets \adv_1(1^n,\vk);~ \sigma\gets \Sign(\sk,0^{\mathsf{cr}\cdot n});~ 
    \\&~~~~~~~~~~~(x,\sigma_i) \gets \adv_2(\vk);~\verify(\vk,x,\sigma_i)]
\end{align*}
Thus, by the unforgeability of the one-time signature scheme we have
\begin{equation*}\AdvRisk_{\problem_n, \mathsf{er}\cdot n,s}(h_n)\leq \negl(n).\end{equation*}
\end{proof}
Now we show that $h_n$ is not robust against computationally unbounded attacks.
\begin{claim}\label{clm:05}
For family $\problem_n$ and any $n \in \N$ we have
$$\AdvRisk_{\problem_n, \ell(n)}(h_n)= 0.5.$$
\end{claim}

\begin{proof}
For any $(([x],\sigma^n, [\vk]), y)$ define $A([x],\sigma^n,[\vk])$ as follows: If $y=1$, $A$ does nothing and outputs $([x],\sigma^n, [\vk])$. If $y=0$, $A$ search all possible signatures to find a signature $\sigma'$ such that $\verify(\vk,\sigma',x) =1 $. It then outputs $([x],(\sigma',\sigma^{n-1}), [\vk])$. Based on the fact that the size of signature is $\ell(n)$, we have $\HD((x, (\sigma',\sigma^{n-1}), [\vk]) , (x,\sigma^n,[\vk])) \leq \ell(n).$ Also, it is clear that $h_n(x, (\sigma',\sigma^{n-1}), [\vk]) = 1$ because the first signature is always a valid signature. Therefore we have
\begin{align*}
    \AdvRisk_{\problem_n, \ell(n)}(h_n)
    &\geq \Pr[(([x],\sigma^n, [\vk]),y)\gets D_n; h(A(([x],\sigma^n, [\vk]))) \neq y ]\\
    &= \Pr[(([x],\sigma^n, [\vk]),y)\gets D_n; 1 \neq y ]\\
    & = 0.5.
\end{align*}
\end{proof}
This concludes the proof of Theorem \ref{thm:comprobhyp3}.
\end{proof}
\section{Average-Case Hardness of \textbf{NP} from  Computational Robustness} \label{sec:NP}

In this section, we show a converse result to those in Section \ref{sec:SepWith}, going from useful computational robustness to deriving computational hardness. Namely, we show that if for there is a learning problem  whose computational risk is noticeably more than its information theoretic risk, then $\NP$ is hard even  on average.

\begin{definition}[Hard samplers for $\NP$] For the following definition, A Boolean formula $\phi$ over some Boolean variables $x_1,\dots,x_k$ is satisfiable, if there is an assignment to $x_1,\dots,x_k \in \bits$, for which $\phi $ evaluates to $1$ (i.e, TRUE). We use some standard canonical encoding of such Boolean formulas and fix it, and we refer to $|\phi|$, the size of $\phi$, as the bit-length of this representation for formula $\phi$. Let $\SAT$ be the language/set  of  all satisfiable Boolean formulas. Suppose $S(1^n,r)$ is a polynomial time randomized algorithm that takes $1^n$ and randomness $r$, runs in time $\poly(n)$, and outputs Boolean formulas of size $\poly(n)$. We call $S$ a \emph{hard (instance) sampler} for $\NP$ if,
\begin{enumerate}
    \item For a negligible function $\negl$ it holds that  $\Pr_{\phi \gets S}[\phi \in \SAT]= 1-\negl(n).$
    \item For every poly-size circuit $\adv$, there is a negligible  function  $\negl$, such that
    $$\Pr_{\phi \gets S, t \gets \adv(\phi)}[ \phi(t)=1]= \negl(n).$$
\end{enumerate}
\end{definition}
The following theorem is stated for computationally robust learning, but the same proof holds for  computationally robust hypotheses as well.
\begin{theorem}[Hardness of $\NP$ from computational robustness] \label{thm:NP}
Let $\problem_n=(\X_n,\Y_n,D_n,\H_n)$  be a learning problem. Suppose there is a (uniform) learning algorithm $L$  for $\problem_n$  such that:
\begin{enumerate}
    \item $L$ is \emph{computationally} robust with risk at most $\alpha$ under $b(n)$-perturbations.
    \item $\AdvRisk_{\problem_n, b(n)}(L)\geq \beta(n)$; i.e., information-theoretic adversarial risk of $L$  is at least $\beta(n)$.  
    \item $\beta(n)-\alpha \geq \eps$ for $\eps= 1/\poly(n)$.
    \item $D_n$ is efficiently samplable by algorithm $S$.
    \item For any $x,x' \in \cX_n$ checking $\metric(x,x')\leq b(n)$ is possible in polynomial time.
\end{enumerate}
Then, there is a hard sampler for $\NP$.
\end{theorem}

Before proving Theorem \ref{thm:NP}, we recall a useful lemma.
The same proof of   amplification of (weak to strong) one-way functions by \citet{yao1982theory} and described in \citep{goldreich2007foundations}, or the parallel repetition of verifiable puzzles \citep{canetti2005hardness,holenstein2011general}  can be used to prove the following lemma.

\begin{lemma}[Amplification of verifiable puzzles] \label{lem:AmpPuz} Suppose $S$ is a distribution over Boolean formulas such that for every poly-size adversary $\adv$, for sufficiently large $n$, it holds that solving the puzzles generated by $S$ are weakly hard. Namely,
$\Pr_{\phi \gets S(1^n,r_1)}[\phi(t)=1; t \gets \adv(\phi)] \leq \eps$
for $\eps=1/\poly(n)$. Then, for any polynomial-size adversary $\adv$, there is a negligible function $\negl$, such that the probability that $\adv$ can simultaneously solve \emph{all} of $k=n/\eps$ puzzles $\phi_1,\dots,\phi_k$ that are independently sampled from $S$ is at most $\negl(n)$. 
\end{lemma}

\begin{proof} (of Theorem \ref{thm:NP}.)
First consider the following sampler $S_1$. (We will modify $S_1$ later on). 
\begin{enumerate}
    \item Sample $m$ examples $\cS \gets D_n^m$.
    \item Run $L$ to get  $h \gets L(\cS)$.
    \item Sample another $(x,y) \gets D_n$
    \item Using the Cook-Levin reduction, get a Boolean formula $\phi=\phi_{h,x,y}$ such that $\phi \in \SAT$, if (1) $\metric(x',x')\leq b(n)$ and (2) $h(x') \neq y$. This is possible because using $h,x,y$, both conditions (1) and (2) are efficiently checkable. 
    \item Output $\phi$.
\end{enumerate} 
By the assumptions of Theorem \ref{thm:NP}, it holds that 
$\Pr_{\phi \gets S_1}[\phi \in \SAT] \geq \beta(n)$
while for any poly-size algorithm $\adv$, it holds that 
$\Pr_{\phi \gets S_1, t \gets \adv(\phi)}[ \phi(t)=1]\leq \alpha$.
So, $S_1$ almost gets the conditions of a hard sampler for $\NP$, but only with a weak sense.

Using standard techniques, we can amplify the $\eps$-gap between $\alpha,\beta(n)$. The algorithm $S_2$ works as follows. (This algorithm assumes the functions $\alpha,\beta(n)$ are efficiently computable, or at least there is an efficiently computable threshold $\tau \in [\alpha+1/\poly(n),\beta(n)-1/\poly(n)]$.)
\begin{enumerate}
    \item For $k=n/\eps^2$, and all $i \in [k]$, get $\phi_i \gets S_1$.
    \item Using the Cook-Levin reduction get a Boolean formula $\phi=\phi_{\phi_1,\dots,\phi_k}$ such $\phi \in \SAT$, if there is a solution to satisfy at least $\tau = (\alpha+\beta(n))/2$ of the formulas $\phi_1,\dots,\phi_k$. More formally, $\phi \in \SAT$, if there is a vector $\ol{t}=(t_1,\dots,t_k)$ such that $|\set{i \colon \phi_i(t_i)=1}|\geq \tau$. This is possible since verifying $\ol{t}$ is efficiently possible.
\end{enumerate}
By the Chernoff-Hoeffding bound, 
\begin{equation*}\Pr_{\phi \gets S_2}[\phi \in \SAT] \geq 1-e^{-(\eps/2)^2 \cdot n/\eps^2} \geq 1-e^{-n/4}.\end{equation*}

Proving the second property of the hard sampler $S$ is less trivial, as it needs an efficient \emph{reduction}. However, we can apply a weak bound here and then use Lemma \ref{lem:AmpPuz}. We first claim that for any poly-size adversary $\adv$, 
\begin{equation} \label{eq:weakHard}
\Pr_{\phi \gets S_2, t \gets \adv(\phi)}[\phi(t)=1] \leq 1-\eps/3.
\end{equation}
To prove Equation \ref{eq:weakHard}, suppose for sake of contradiction that there is such adversary $\adv$. We can use $\adv$ and solve $\phi' \gets S_1$ with probability more than $\alpha+\Omega(\eps)$ which is a contradiction. Given $\phi'$, The reduction is as follows.
\begin{enumerate}
    \item Choose $i \gets [k]$ at random.
    \item Sample $k-1$ instances $\phi_1,\dots,\phi_{i-1}, \phi_{i+1},\dots,\phi_k \gets S_1$ independently at random.
    \item Let $\phi_i =\phi'$.
    \item Ask $\adv$ to solve $\phi_{\phi_1,\dots,\phi_k}$, and if $\adv$'s answer gave a solution for $\phi_i=\phi'$, output this solution.
\end{enumerate}
Since $\adv$ cannot guess $i$, a simple argument shows that the above reduction succeeds with probability $\alpha+\eps/2-\eps/3 = \alpha + \eps/6$.
Now that we have a puzzle generator $S_2$ that has satisfiable puzzles with probability $1-\negl(n)$ and efficient algorithms can solve its solutions by probability at most $\eps/2$, using Lemma \ref{lem:AmpPuz}, we can use another direct product and design sampler $S$ that samples $2n/\eps$ independent instances from $S_2$ and asks for solutions to all of them. Because we already established that $\Pr_{\phi \gets S_2}[\phi \in \SAT] \geq 1-\negl(n)$, the puzzles sampled by $S$ are also satisfiable by probability $1-n \cdot \negl(n) = 1-\negl(n)$, but efficient algorithms  can still find the solution only with  probability that is $\negl(n)$.
\end{proof}

\section{Conclusion}
\label{sec:conc}
The assumption  of computationally-bounded adversaries has been the key to modern cryptography. In fact, without this assumption modern cryptographic primitives would not be possible. This paper investigates whether this assumption helps in the context of robust learning and demonstrates that is indeed the case (i.e., computational hardness can be leveraged in robust learning).  We hope that this work is the first-step in leveraging computational hardness in the context of robust learning. 

Several intriguing questions remain, such as:

\begin{itemize}
\item Our Construction \ref{thm:comprobhyp} takes a natural learning problem, but then it modifies it. Can computational robustness be achieved for natural problems, such as image classification?  
\item Theorem \ref{thm:NP} shows that computational hardness is necessary for nontrivial computational robustness. However, this does not still mean we can get
cryptographic primitives back from such problems. Can we  obtain cryptographically useful primitives, such as \emph{one-way functions}, from such computational robustness? 
\end{itemize}



\appendix
\section{Useful Tools}\label{sec:tools}
Here, we define the notions of one-way function, one-time signature and error correcting code.

\begin{definition}[One-way function] A function $f\colon \set{0,1}^* \to \set{0,1}^*$ is one-way if it can be computed in polynomial time and the inverse of $f$ is hard to compute. Namely, there is a polynomial time algorithm $M$ such that
$$\Pr[x\gets \set{0,1}^n; M(x)=f(x)] =1$$
and for any polynomial time algorithm $\adv$ there is a negligible function $\negl(\cdot)$ such that we have,
$$\Pr[x\gets \set{0,1}^n; y=f(x); f(\adv(y))=x] \leq \negl(|x|).$$
\end{definition}
The assumption that one-way functions exist is standard and omnipresent in cryptography as a \emph{minimal} assumption, as many cryptographic tasks imply the existence of OWFs \citep{goldreich2007foundations,katz2014introduction}.

\begin{definition}[One-time signature schemes] \label{def:sign} A one-time signature scheme $S=(\KGen, \Sign, \verify)$ consists of three probabilistic
polynomial-time  algorithms as follows:
\begin{itemize}
    \item $\KGen(1^\lambda)\footnote{\text{By $1^\lambda$} we mean an string of bits of size $\lambda$ that is equal to 1 at each location. Note that $\lambda$ is the security parameter that controls the security of the scheme. As $\lambda$ increases  the task of finding a forgery for a signature becomes harder.} \to (\sk,\vk)$
    \item $\Sign(\sk, m) \to \sigma$
    \item $\verify(\vk, \sigma, m) \to \set{0,1}$
\end{itemize}
which satisfy the following properties:
\begin{enumerate}
    \item {\bf Completeness:} For every $m$
    \begin{align*}
        \Pr[&(sk,vk)\gets \KGen(1^\lambda); \sigma \gets \Sign(sk, m);\\&~~~~\verify(\vk, \sigma, m) = 1] =1.
    \end{align*}
    \item {\bf Unforgeability:} For every positive polynomial $s$, for every $\lambda$ and every pair of circuits $(A_1,A_2)$ with size $s(\lambda)$ the following probability is negligible in $\lambda$:
    \begin{align*}
        \Pr[&(\sk,\vk)\gets \KGen(1^\lambda); \\
        &(m,st)\gets A_1(1^\lambda, \vk);\\ & \sigma \gets \Sign(\sk,m); \\&(m', \sigma')\gets A_2(1^\lambda, \vk, st, m,\sigma);\\&  m\neq m' \wedge \verify(\vk, \sigma',m')=1] \leq \negl(\lambda).
    \end{align*}
\end{enumerate}
\end{definition}

\begin{definition}[Error correction codes] An error correction code with code rate $\alpha$ and error rate $\beta$ consists of two algorithms $\encode$ and $\decode$ as follows.
\begin{itemize}
    \item The encode algorithm $\encode$ takes a Boolean string $m$ and outputs a Boolean string $c$ such that $|c| =  |m|/\alpha$.  
    \item The decode algorithm $\decode$ takes a Boolean string $c$ and outputs either $\bot$ or a Boolean string $m$. It holds that for all $m\in\set{0,1}^*$, $c=\encode(m)$ and
$c'$ where $ \HD(c,c') \leq {\beta\cdot |c|}$, it holds that $ \decode(c') = m.$
\end{itemize}

\end{definition}




\end{document}